\documentclass{article}
\usepackage[utf8]{inputenc}

\usepackage{paper_header}

\title{Improving the Efficiency of Off-Policy Reinforcement Learning\\by Accounting for Past Decisions}

\author{
    Brett Daley\\
    Khoury College of Computer Sciences\\
    Northeastern University\\
    Boston, MA 02115, USA\\
    \texttt{daley.br@northeastern.edu}
    \and
    Christopher Amato\\
    Khoury College of Computer Sciences\\
    Northeastern University\\
    Boston, MA 02115, USA\\
    \texttt{c.amato@northeastern.edu}
}
\date{}

\begin{document}

\maketitle

\begin{abstract}
    \noindent
    Off-policy learning from multistep returns is crucial for sample-efficient reinforcement learning, particularly in the experience replay setting now commonly used with deep neural networks.
    Classically, off-policy estimation bias is corrected in a \textit{per-decision} manner:
    past temporal-difference errors are re-weighted by the instantaneous Importance Sampling (IS) ratio (via eligibility traces) after each action.
    Many important off-policy algorithms such as Tree Backup and Retrace rely on this mechanism along with differing protocols for truncating (``cutting'') the ratios (``traces'') to counteract the excessive variance of the IS estimator.
    Unfortunately, cutting traces on a per-decision basis is not necessarily efficient;
    once a trace has been cut according to local information, the effect cannot be reversed later, potentially resulting in the premature truncation of estimated returns and slower learning.
    In the interest of motivating efficient off-policy algorithms, we propose a multistep operator that permits arbitrary past-dependent traces.
    We prove that our operator is convergent for policy evaluation, and for optimal control when targeting greedy-in-the-limit policies.
    Our theorems establish the first convergence guarantees for many existing algorithms including Truncated IS, Non-Markov Retrace, and history-dependent TD($\lambda$).
    Our theoretical results also provide guidance for the development of new algorithms that jointly consider multiple past decisions for better credit assignment and faster learning.
\end{abstract}

\section{Introduction}

Reinforcement learning concerns an agent interacting with its environment through trial and error to maximize its expected cumulative reward \cite{sutton1998reinforcement}.
One of the great challenges of reinforcement learning is the temporal credit assignment problem \cite{sutton1984temporal}:
upon the receipt of a reward, which past action(s) should be held responsible and, hence, be reinforced?
Basic temporal-difference (TD) methods \cite[e.g.][]{watkins1989learning, rummery1994line} assign credit to the immediately taken action, bootstrapping from previous experience to learn long-term dependencies.
This process requires a large number of repetitions to generate effective behaviors from rewards, motivating research into \textit{multistep} return estimation
in which credit is distributed among multiple (even infinitely many) past actions according to some eligibility rule \cite[e.g.][]{sutton1988learning, watkins1989learning, sutton1998reinforcement}.
This can lead to significantly faster learning~\cite{sutton1998reinforcement}.

One challenge of multistep estimators is that they generally have higher variance than 1-step estimators~\cite{kearns2000bias}.
This is exacerbated in the \textit{off-policy} setting, where environment interaction is conducted according to a behavior policy that differs from the target policy for which returns are being estimated.
The discrepancy between the two policies manifests mathematically as bias in the return estimation, which can be detrimental to learning if left unaddressed.
Despite these challenges, off-policy learning is important for exploration and sample efficiency, especially when combined with experience replay \cite{lin1992self}, which has gained popularity due to deep reinforcement learning \cite[e.g.][]{mnih2015human}.
The canonical bias-correction technique is Importance Sampling (IS) wherein the bias due to the differing policies is eliminated by the product of their probability ratios.
Although IS theoretically resolves the off-policy bias, it can suffer from extreme variance that makes the algorithm largely untenable in practice.

Directly managing the variance of the IS estimator has been a fruitful avenue for developing efficient off-policy algorithms.
Past work has focused on modifying the individual IS ratios to reduce the variance of the full update:
e.g.\ Tree Backup \cite{precup2000eligibility}, Q($\lambda$) with Off-Policy Corrections \cite{harutyunyan2016q}, Retrace \cite{munos2016safe}, ABQ \cite{mahmood2017multi}, and C-trace~\cite{rowland2020adaptive}.
All of these methods can be implemented online with \textit{per-decision} rules \cite{precup2000eligibility} that determine how much to decay the IS ratio according to the current state-action pair---a process commonly called ``trace cutting'' in reference to the first method to utilize this technique, Watkins' Q($\lambda$) \cite{watkins1989learning, sutton1998reinforcement, munos2016safe}.
The re-weighted TD error is then broadcast to previous experiences using eligibility traces \cite{klopf1972brain, barto1983neuronlike, sutton1984temporal, sutton1998reinforcement}.
These algorithms are \textit{Markov} in the sense that each iterative off-policy correction depends on only the current state-action pair.
One issue with this strategy is that it may lead to suboptimal decisions, since cutting a trace cannot be reversed later, even as new information is ascertained.
In contrast, a \textit{non-Markov} method could examine an entire trajectory of past state-action pairs to make globally better decisions regarding credit assignment.

Indeed, some existing off-policy methods already conduct offline bias correction in a non-Markov manner.
Perhaps the simplest example is Truncated IS where the IS ratio products are pre-calculated offline and then clipped to some finite value (see Section~\ref{sect:example_algorithms}).
More recently, \cite{munos2016safe} suggested a recursive variant of Retrace that automatically relaxes the clipping bound when its historical trace magnitude becomes small;
the authors conjecture that this could lead to more efficient learning.
To our knowledge, no theoretical analysis has been conducted on non-Markov algorithms such as these,
meaning that their convergence properties are unknown, and the space of possible algorithms has not been fully explored.

To better understand the behavior of these non-Markov off-policy algorithms, and to support new discoveries of efficient algorithms, we propose a multistep operator $\mathcal{M}$ that accounts for arbitrary dependencies on past decisions.
Our operator is a significant generalization of the $\mathcal{R}$ operator proposed by \cite{munos2016safe} in that it allows an arbitrary, time-dependent weight for each TD error that may generally depend on the history of the Markov Decision Process (MDP).
We prove that our operator converges for policy evaluation and for control provided that the chosen weights never exceed the true IS ratio product.
We also remove the assumptions of increasingly greedy policies and pessimistic initialization used by \cite{munos2016safe} in the control setting.
Finally, we discuss practical considerations when implementing algorithms described by our operator.
Our results presented here provide new insights into off-policy credit assignment and open avenues for efficient empirical methods---with particularly interesting opportunities for deep reinforcement learning and experience replay.

\section{Preliminaries}

We consider MDPs of the form $(\mathcal{S}, \mathcal{A}, P, R, \gamma)$.
$\mathcal{S}$ and $\mathcal{A}$ are finite sets of states and actions, respectively.
Letting $\Delta X$ denote the set of distributions over a set $X$, then
${P \colon \mathcal{S} \times \mathcal{A} \to \Delta \mathcal{S}}$
is the transition function,
$R \colon \mathcal{S} \times \mathcal{A} \to \mathbb{R}$
is the reward function, and $\gamma \in [0,1]$ is the discount factor.

A policy $\pi \colon \mathcal{S} \to \Delta \mathcal{A}$ determines the agent's probability of selecting a given action in each state.
A Q-function $Q \colon \mathcal{S} \times \mathcal{A} \to \mathbb{R}$ represents the agent's estimate of the expected return achievable from each state-action pair.
For a policy $\pi$, we define the operator $P_\pi$ as
\begin{equation*}
    (P_\pi Q)(s,a) \coloneqq \sum_{s' \in \mathcal{S}} \sum_{a' \in \mathcal{A}} P(s'|s,a) \pi(a'|s') Q(s',a')
    .
\end{equation*}
As a shorthand notation, it is useful to represent Q-functions and the reward function as vectors in $\mathbb{R}^n$ where $n = |\mathcal{S} \times \mathcal{A}|$.
We define $e$ to be the vector whose components are equal to $1$.
Linear operators such as $P_\pi$ can hence be interpreted as $n \times n$ square matrices that multiply these vectors, with repeated application corresponding to exponentiation:
e.g.\ $P_\pi^t Q = P_\pi (P_\pi^{t-1} Q)$.
All inequalities in our work should be interpreted element wise when involving vectors or matrices.
Finally, we let $\norm{A} \coloneqq \norm{A}_\infty$ for a matrix or vector $A$.

In the \textit{policy evaluation} setting, we seek to estimate the expected discounted return for policy $\pi$, given by
\begin{equation*}
    Q^\pi \coloneqq \sum_{t \geq 0} \gamma^t P_\pi^t R
    .
\end{equation*}
$Q^\pi$ is the unique fixed point of the Bellman operator $T_\pi Q \coloneqq R + \gamma P_\pi Q$ \cite{bellman1966dynamic}.
In the \textit{control} setting, we seek to estimate the expected return under the optimal policy $\pi^*$, denoted by $Q^*$.
Letting $\Pi$ be the set of all possible policies, then $Q^*$ is the unique fixed point of the Bellman optimality operator $T Q \coloneqq R + \gamma \max\limits_{\pi' \in \Pi} P_{\pi'} Q$.

\subsection{Multistep Off-Policy Operators}

In our work, we are particularly interested in the \textit{off-policy} learning case, where trajectories of the form
$((s, a), (s_1, a_1), (s_2, a_2), \dots)$,
$(s_k,a_k) \in \mathcal{S} \times \mathcal{A}$,
are generated by interacting with the MDP using a behavior policy $\mu \in \Pi$, $\mu \neq \pi$.
Let $\mathcal{F}_t$ denote the first $t+1$ terms of this sequence.
We define the TD error for policy $\pi$ at timestep $t$ as
\begin{equation*}
    \delta^\pi_t \coloneqq r_t + \gamma \sum_{a' \in \mathcal{A}} \pi(a'|s_{t+1}) Q(s_{t+1}, a') - Q(s_t, a_t)
    .
\end{equation*}
Let $\rho_k \coloneqq \frac{\pi(a_k|s_k)}{\mu(a_k|s_k)}$ for brevity.
\cite{munos2016safe} introduced the following off-policy operator:
\begin{equation}
    \label{eq:retrace}
    (\mathcal{R}Q)(s,a) \coloneqq Q(s,a) + \E_\mu \Big[ \sum_{t \geq 0} \gamma^t \Big(\prod_{k=1}^t c(s_k,a_k)\Big) \delta^\pi_t \Big]
    ,
\end{equation}
where $c(s_k,a_k) \in [0, \rho_k]$ are arbitrary nonnegative coefficients, or \textit{traces}.
When $c(s_k,a_k) < \rho_k$, we say that the trace has been (partially) \textit{cut}.
Note that the traces are Markov in the sense that they depend only on the state-action pair $(s_k,a_k)$ and are otherwise independent of $\mathcal{F}_k$.
In other words, the traces for $\mathcal{R}$ can be calculated \textit{per decision} \cite{precup2000eligibility}, thereby permitting an efficient online implementation with eligibility traces.

While per-decision traces are convenient from a computational perspective, they require making choices about how much to cut a trace without direct consideration of the past.
This can lead to locally optimal decisions;
for example, if the trace is cut by setting $c(s_k,a_k) = 0$ at some timestep, then the effect cannot be reversed later.
Regardless of whatever new experiences are encountered by the agent, experiences before timestep $k$ will be ineligible for credit assignment, resulting in an opportunity cost.
In fact, this exact phenomenon is why Watkins' Q($\lambda$) \cite{watkins1989learning} often learns more slowly than Peng's Q($\lambda$) \cite{peng1994incremental} even though the former avoids off-policy bias \cite{sutton1998reinforcement, daley2019reconciling, kozuno2021revisiting}.
The same effect (but to a lesser extent) impacts Tree Backup and Retrace where $c(s_k,a_k) \leq 1$ always, implying that the eligibilities for past experiences can never increase.

For this reason, it may be desirable to expend some additional computation in exchange for better decisions regarding credit assignment.
Our principal contribution is the proposal and analysis of the following off-policy operator that encompasses this possibility:
\begin{equation}
    \label{eq:definition}
    (\mathcal{M}Q)(s,a) \coloneqq Q(s,a) + \E_\mu \Big[ \sum_{t \geq 0} \gamma^t \beta(\mathcal{F}_t) \delta^\pi_t \Big]
    ,
\end{equation}
where each $\beta(\mathcal{F}_t)$ is an arbitrary nonnegative coefficient that generally depends on the history $\mathcal{F}_t$.
By analogy to TD(0) \cite{sutton1988learning}, we define $\beta(\mathcal{F}_0) \coloneqq 1$.
The principal goal of Section~\ref{sect:analysis} is to characterize the values of $\beta(\mathcal{F}_t)$ for $t \geq 1$ that lead to convergence in the policy evaluation and control settings.

The major analytical challenge of our operator, and its main novelty, is the complex dependence on the sequence $\mathcal{F}_t$.
Our operator is therefore inherently \textit{non-Markov}, in spite of the fact that the MDP dynamics are Markov by definition.
Mathematically, this makes the $\mathcal{M}$ operator difficult to analyze, as the terms in the series
$1 + \beta(\mathcal{F}_1) + \beta(\mathcal{F}_2) + \cdots$
generally share no common factors that would allow a per-decision update for eligibility traces.
Nevertheless, removing the Markov assumption is necessary to understand important existing algorithms (see Section~\ref{sect:example_algorithms}) while also paving the way for new credit assignment algorithms.

A special case arises when $\beta(\mathcal{F}_t)$ does factor into traces:
that is, $\beta(\mathcal{F}_t) = \prod_{k=1}^t c(s_k,a_k)$ for every history $\mathcal{F}_t$.
Equation (\ref{eq:definition}) therefore reduces to equation (\ref{eq:retrace}), taking us back to the Markov setting studied by \cite{munos2016safe}.
As such, our $\mathcal{M}$ operator subsumes the $\mathcal{R}$ operator as a specific case, and our later theoretical developments have implications for Retrace and related off-policy algorithms.
We discuss these implications in Section~\ref{sect:discussion}.

\subsection{Example Algorithms}
\label{sect:example_algorithms}

By weighting each TD error with an arbitrary history-dependent coefficient $\beta(\mathcal{F}_t)$, the $\mathcal{M}$ operator is remarkably general, but this also makes it abstract.
To help concretize it, we describe several existing off-policy algorithms below and how they can be described in the notation of our operator.
This should aid with understanding and illuminate cases where our operator is necessary to explain convergence.

\renewcommand{\thefootnote}{\fnsymbol{footnote}}  % Dagger for non-Markov algorithms

\textbf{Importance Sampling (IS):
$\beta(\mathcal{F}_t) = \prod_{k=1}^t \rho_k$.}\\
The standard approach for correcting off-policy experience.
Although it is the only unbiased estimator in this list, it suffers from high variance when $\mu(a_k|s_k) \approx 0$ that makes it difficult to utilize effectively in practice.

\textbf{Q($\lambda$) with Off-Policy Corrections \cite{harutyunyan2016q}:
$\beta(\mathcal{F}_t) = \prod_{k=1}^t \lambda$.}\\
A straightforward algorithm that decays the TD errors by a fixed constant $\lambda \in [0,1]$.
In the on-policy ($\mu = \pi$) policy evaluation case, this is equivalent to the TD($\lambda$) extension \cite{sutton1988learning} of Expected Sarsa \cite{sutton1998reinforcement}.
The algorithm does not require explicit knowledge of the behavior policy $\mu$, which is desirable;
however, it is not convergent when $\pi$ and $\mu$ differ too much, and this criterion is restrictive in practice.

\textbf{Tree Backup (TB) \cite{precup2000eligibility}:
$\beta(\mathcal{F}_t) = \prod_{k=1}^t \lambda \pi(a_k|s_k)$.}\\
A method that automatically cuts traces according to the product of probabilities under the target policy $\pi$, which forms a conservative lower bound on the IS ratio product.
As a result, TB converges for any behavior policy $\mu$, but it is not efficient since traces are cut excessively---especially in the on-policy case.

\textbf{Retrace \cite{munos2016safe}:
$\beta(\mathcal{F}_t) = \prod_{k=1}^t \lambda \min (1, \rho_k)$.}\\
A convergent algorithm for arbitrary policies $\pi$ and $\mu$ that remains efficient in the on-policy case because it does not cut traces (if $\lambda = 1$);
however, the fact that $\beta(\mathcal{F}_t)$ is monotone non-increasing can cause the trace products to decay too quickly in practice \cite{mahmood2017multi, rowland2020adaptive}.

\textbf{\footnotemark[2]Non-Markov Retrace \cite{munos2016safe}:
$\beta(\mathcal{F}_t) = \lambda \min (1, \beta(\mathcal{F}_{t-1}) \rho_t )$.}\\
A modification to Retrace proposed by \cite{munos2016safe} and conjectured to lead to faster learning.
It permits trace values larger than $1$ by relaxing the clipping bound when the historical trace product $\beta(\mathcal{F}_{t-1})$ is small (see Section~\ref{sect:convergence_for_specific}).
The algorithm is recursive, which makes it difficult to analyze, and its convergence for control is an open question.

\textbf{\footnotemark[2]Truncated Importance Sampling \cite{ionides2008truncated}:
$\beta(\mathcal{F}_t) = \min (d_t, \prod_{k=1}^t \rho_k )$.}\\
A simple but effective method to combat the variance of the IS estimator.
For any value of $d_t \geq 0$, the obtained variance is finite.
Variations of this algorithm have been applied in the reinforcement learning literature \cite[e.g.][]{uchibe2004competitive, wawrzynski2007truncated, wawrzynski2009real, wang2016sample}, but to our knowledge its convergence in an MDP setting has not been studied.

\footnotetext[2]{Denotes a non-Markov algorithm that must be modeled by our $\mathcal{M}$ operator defined in equation (\ref{eq:definition}).}

\section{Analysis}
\label{sect:analysis}

In this section, we study the convergence properties of the $\mathcal{M}$ operator for policy evaluation and control.
It will be convenient to re-express equation (\ref{eq:definition}) in operator notation in our following analysis.
First, we can define an operator $B_t$ such that
\begin{equation*}
    (B_t Q)(s,a) \coloneqq \beta(\mathcal{F}_{t-1} \cup (s,a)) Q(s,a)
    .
\end{equation*}
Note that $B_t$ can be interpreted as a diagonal matrix where each element along the main diagonal is equal to $\beta(\mathcal{F}_t)$ after hypothetically selecting the corresponding state-action pair $(s,a)$ in history $\mathcal{F}_{t-1}$.
Furthermore, by our earlier definition of $\beta(\mathcal{F}_0) = 1$, we must have $B_0 = I$, the identity matrix.
With this, we can write the $\mathcal{M}$ operator as
\begin{equation}
    \label{eq:operator}
    \mathcal{M}Q = Q + \sum_{t \geq 0} \gamma^t P_\mu^t B_t (T_\pi Q - Q)
    .
\end{equation}

\subsection{Policy Evaluation}

In this setting, we seek to estimate $Q^\pi$ for a fixed target policy $\pi \in \Pi$ from interactions conducted according to a fixed behavior policy $\mu \in \Pi$.
Specifically, our goal is to prove that repeated application of the $\mathcal{M}$ operator to an arbitrarily initialized $Q$ converges to $Q^\pi$.
We start by proving the following lemma:
\begin{lemma}
    \label{lemma:diff}
    Let $\mathcal{M}$ be the operator defined in (\ref{eq:operator}).
    $Q^\pi$ is a fixed point of $\mathcal{M}$;
    the difference between $\mathcal{M}Q$ and $Q^\pi$ is given by
    \begin{equation}
        \label{eq:operator_rewrite}
        \mathcal{M}Q - Q^\pi
        = \sum_{t \geq 1} \gamma^t P_\mu^{t-1} (B_{t-1} P_\pi - P_\mu B_t) (Q - Q^\pi)
        .
    \end{equation}
\end{lemma}
We present the proof in Appendix~\ref{app:lemma_diff}.
With this, we can introduce our main theorem for policy evaluation:
\begin{theorem}
    \label{theorem:evaluation}
    Suppose that
    $B_t \leq \prod_{k=1}^t \rho_k$
    for any history
    $\mathcal{F}_{t-1} \cup (s_t,a_t)$.
    The operator $\mathcal{M}$ defined in (\ref{eq:operator}) is a contraction mapping with $Q^\pi$ as its unique fixed point.
    That is,
    \begin{equation*}
        \norm{\mathcal{M}Q - Q^\pi} \leq \gamma \norm{Q - Q^\pi}
        ,
    \end{equation*}
    and consequently $\lim\limits_{k \to \infty} \mathcal{M}^k Q = Q^\pi$ for any $Q$.
\end{theorem}

\begin{proof}
    Lemma~\ref{lemma:diff} already established that $Q^\pi$ is a fixed point of $\mathcal{M}$.
    We will show that $\mathcal{M}$ is a contraction mapping, which will guarantee that $Q^\pi$ is its \textit{unique} fixed point.

    Let
    $A \coloneqq \sum_{t \geq 1} \gamma^t P_\mu^{t-1} (B_{t-1} P_\pi - P_\mu B_t)$
    and rewrite (\ref{eq:operator_rewrite}) as
    $\mathcal{M}Q - Q^\pi = A(Q - Q^\pi)$.
    By our assumption that $B_t$ is diagonal and
    $\smash{B_t \leq \prod_{k=1}^t \rho_k}$,
    the matrix $A$ has nonnegative elements
    (since $B_{t-1} P_\pi - P_\mu B_t \geq 0$).
    To complete the proof, we can show that the row sums of $A$ are never greater than $\gamma$.
    Equivalently, we will show that $A e \leq \gamma e$:
    \begin{align*}
        \allowdisplaybreaks
        A e
        &= \sum_{t \geq 1} \gamma^t P_\mu^{t-1} (B_{t-1} P_\pi - P_\mu B_t) e \\
        &= \sum_{t \geq 1} \gamma^t P_\mu^{t-1} (B_{t-1} e - P_\mu B_t e) \\
        &= \gamma \sum_{t \geq 0} \gamma^t P_\mu^t B_t e - \sum_{t \geq 1} \gamma^t P_\mu^t B_t e \\
        &= \gamma (e + S) - S \\
        &= \gamma e - (1-\gamma) S
        ,
    \end{align*}
    where we have let
    $S \coloneqq \sum_{t \geq 1} \gamma^t P_\mu^t B_t e$.
    Because $B_t \geq 0$, we have that $S \geq 0$ and hence $A e \leq \gamma e$;
    thus, ${A (Q - Q^\pi)}$ is a vector whose components each comprise a nonnegative-weighted combination of the components of $Q - Q^\pi$ where the weights add up to at most $\gamma$.
    This implies that 
    \begin{equation*}
        \norm{\mathcal{M}Q - Q^\pi} \leq \gamma \norm{Q - Q^\pi}
    \end{equation*}
    and the operator $\mathcal{M}$ is a contraction mapping.
    Its fixed point $Q^\pi$ established by Lemma~\ref{lemma:diff} must therefore be unique,
    which implies that $\lim\limits_{k \to \infty} \mathcal{M}^k Q = Q^\pi$ for any $Q$ when $\gamma < 1$.
\end{proof}

We conclude with some remarks about the interpretation of our result.
Note that when $\beta(\mathcal{F}_t) = 0$ for $t \geq 1$, we have the slowest contraction rate of $\gamma$, which corresponds to TD(0).
When $\beta(\mathcal{F}_t) = \prod_{k=1}^t \rho_k$, we get an optimal contraction rate of $0$, which corresponds to the standard IS estimator.
(Of course, this is only in expectation.)
Between these two extremes lies a vast spectrum of convergent algorithms, which includes the examples mentioned in Section~\ref{sect:example_algorithms} as well as an infinitude of other possibilities.
In Section~\ref{sect:practical_concerns}, we discuss practical considerations for choosing $\beta(\mathcal{F}_t)$ and how to efficiently implement history-dependent algorithms.

\subsection{Control}

We now consider the more-difficult setting of control.
Given a sequence of target policies $(\pi_k)_{k \geq 0}$, $\pi_k \in \Pi$, and a sequence of behavior policies $(\mu_k)_{k \geq 0}$, $\mu_k \in \Pi$,
we aim to show that the sequence of Q-functions $(Q_k)$ given by $Q_{k+1} \coloneqq \mathcal{M}_k Q_k$ converges to $Q^*$.
(Here, $\mathcal{M}_k$ is the $\mathcal{M}$ operator defined for $\pi_k$ and $\mu_k$.)

Unlike \cite{munos2016safe}, we do not assume that the target policies are \textit{increasingly} greedy with respect to $(Q_k)$.
Instead, we require only that the policies become greedy in the limit.
We say that the sequence of policies $(\pi_k)$ is greedy in the limit if and only if
$T_{\pi_k} Q_k \to T Q_k$ as $k \to \infty$.
Intuitively, this means that each policy need not be greedier than its predecessors, so long as the sequence of policies eventually does become greedy.
We discuss the significance of this assumption more in Section~\ref{sect:theoretical_contributions}.

Let $C \coloneqq \sum_{t \geq 0} \gamma^t P_\mu^t B_t$.
We first rewrite the $\mathcal{M}$ operator in (\ref{eq:operator}) concisely as
\begin{equation}
    \label{eq:operator_rewrite2}
    \mathcal{M}Q = Q + C(T_\pi Q - Q)
    .
\end{equation}
The following lemma will be useful to show that $\mathcal{M}$ remains a contraction mapping in this new setting:
\begin{lemma}
    \label{lemma:matrix_norm_bound}
    For any policy $\pi \in \Pi$,
    \begin{equation}
        \norm{I - C (I - \gamma P_\pi)} \leq \gamma
        .
    \end{equation}
\end{lemma}
Once again, we present the proof in Appendix~\ref{app:matrix_norm_bound}.
We now arrive at our main theoretical result for control:
\begin{theorem}
    \label{theorem:control}
    Consider a sequence of target policies $(\pi_k)$ and a sequence of arbitrary behavior policies $(\mu_k)$.
    Let $Q_0$ be an arbitrarily initialized Q-function and define the sequence $Q_{k+1} \coloneqq \mathcal{M}_k Q_k$ where $\mathcal{M}_k$ is the operator defined in (\ref{eq:operator}) for $\pi_k$ and $\mu_k$.
    Assume that $(\pi_k)$ is greedy in the limit and let $\epsilon_k \in [0,1]$ be the smallest constant such that
    $T_{\pi_k} Q_k \geq T Q_k - \epsilon_k \norm{Q_k} e$.
    Then,
    \begin{equation*}
        \norm{Q_{k+1} - Q^*} \leq \gamma \norm{Q_k - Q^*} + \frac{\epsilon_k}{1-\gamma} \norm{Q_k}
        ,
    \end{equation*}
    and consequently $\lim\limits_{k \to \infty} Q_k = Q^*$.
\end{theorem}

\begin{proof}
    \textbf{Part 1: Upper bound on $Q_{k+1} - Q^*$.}
    First, note the following inequality:
    \begin{equation*}
        T_{\pi_k} Q_k - T Q^*
        = \gamma P_{\pi_k} Q_k - \gamma \max\limits_{\pi' \in \Pi} P_{\pi'} Q_k
        \leq \gamma P_{\pi_k} (Q_k - Q^*)
        .
    \end{equation*}
    Then, from (\ref{eq:operator_rewrite2}), we can deduce that
    \begin{align}
        \allowdisplaybreaks
        \nonumber
        Q_{k+1} - Q^*
        &= Q_k - Q^* + C(T_{\pi_k} Q_k - Q_k) \\
        \label{eq:start_here}
        &= (I - C)(Q_k - Q^*) + C(T_{\pi_k} Q_k - Q^*) \\
        \nonumber
        &= (I - C)(Q_k - Q^*) + C(T_{\pi_k} Q_k - T Q^*) \\
        \nonumber
        &\leq (I - C)(Q_k - Q^*) + \gamma C P_{\pi_k} (Q_k - Q^*) \\
        \label{eq:upper_bound}
        &= (I - C(I - \gamma P_{\pi_k}))(Q_k - Q^*)
        .
    \end{align}
    \textbf{Part 2: Lower bound on $Q_{k+1} - Q^*$.}
    Note the following inequality:
    \begin{equation*}
        TQ_k - TQ^*
        \geq T_{\pi^*} Q_k - TQ^*
        = \gamma P_{\pi^*} (Q_k - Q^*)
        .
    \end{equation*}
    Additionally, for each policy $\pi_k$, there exists some $\epsilon_k \in [0,1]$ such that
    $T_{\pi_k} Q_k \geq T Q_k - \epsilon_k \norm{Q_k} e$.
    (The inequality is vacuously true when $\epsilon_k = 1$, but recall that Theorem~\ref{theorem:control} defines $\epsilon_k$ to be as small as possible.)
    Starting again from equation (\ref{eq:start_here}),
    and noting that the elements of $C$ are nonnegative,
    \begin{align}
        \allowdisplaybreaks
        \nonumber
        Q_{k+1} - Q^*
        &= (I - C)(Q_k - Q^*) + C(T_\pi Q - Q^*) \\
        \nonumber
        &\geq (I - C)(Q_k - Q^*) + C(TQ_k - Q^*) - \epsilon_k \norm{Q_k} C e \\
        \nonumber
        &= (I - C)(Q_k - Q^*) + C(TQ_k - TQ^*) - \epsilon_k \norm{Q_k} C e \\
        \nonumber
        &\geq (I - C)(Q_k - Q^*) + \gamma C P_{\pi^*} (Q_k - Q^*) - \epsilon_k \norm{Q_k} C e \\
        \label{eq:lower_bound}
        &= (I - C(I - \gamma P_{\pi^*}))(Q_k - Q^*) - \epsilon_k \norm{Q_k} C e
    \end{align}
    \textbf{Part 3: Conclusion.}
    When
    $Q_{k+1} - Q^* \geq 0$,
    inequality (\ref{eq:upper_bound}) and Lemma~\ref{lemma:matrix_norm_bound} imply
    \begin{equation}
        \label{eq:pos_bound}
        \norm{Q_{k+1} - Q^*}
        \leq \norm{I - C(I - \gamma P_{\pi_k})} \norm{Q_k - Q^*}
        \leq \gamma \norm{Q_k - Q^*}
        .
    \end{equation}
    Additionally, when
    $Q_{k+1} - Q^* \leq 0$,
    inequality (\ref{eq:lower_bound}) and Lemma~\ref{lemma:matrix_norm_bound} imply
    \begin{align}
        \nonumber
        \norm{Q_{k+1} - Q^*}
        &\leq \norm{I - C(I - \gamma P_{\pi^*})} \norm{Q_k - Q^*} + \epsilon_k \norm{Q_k} \norm{C} \\
        \label{eq:neg_bound}
        &\leq \gamma \norm{Q_k - Q^*} + \frac{\epsilon_k}{1-\gamma} \norm{Q_k}
        ,
    \end{align}
    because
    $\norm{C} \leq \sum_{t \geq 0} \gamma^t \norm{P_\mu^t B_t} \leq (1-\gamma)^{-1}$.
    Since (\ref{eq:neg_bound}) is more conservative than (\ref{eq:pos_bound}), its bound holds regardless of the sign of
    $Q_{k+1} - Q^*$.
    It remains to be shown that this bound implies convergence to $Q^*$.
    Observe that
    \begin{align*}
        \gamma \norm{Q_k - Q^*} + \frac{\epsilon_k}{1-\gamma} \norm{Q_k}
        &\leq \gamma \norm{Q_k - Q^*} + \frac{\epsilon_k}{1-\gamma} (\norm{Q_k - Q^*} + \norm{Q^*}) \\
        &= \left( \gamma + \frac{\epsilon_k}{1-\gamma} \right) \norm{Q_k - Q^*} + \frac{\epsilon_k}{1-\gamma} \norm{Q^*}
        .
    \end{align*}
    Our assumption of greedy-in-the-limit policies tells us that $\epsilon_k \to 0$ as $k \to \infty$;
    there must exist some iteration $k'$ such that
    $\epsilon_{k} \leq \frac{1}{2} (1-\gamma)^2$ for all $k \geq k'$.
    Therefore, for $k \geq k'$,
    \begin{equation*}
        \norm{Q_{k+1} - Q^*}
        \leq \frac{1 + \gamma}{2} \norm{Q_k - Q^*} + \frac{\epsilon_k}{1-\gamma} \norm{Q^*}
        .
    \end{equation*}
    If $\gamma < 1$,
    then
    $\frac{1}{2} (1 + \gamma) < 1$.
    Since $\norm{Q^*}$ is finite, we conclude that $Q_k \to Q^*$ as $\epsilon_k \to 0$.
\end{proof}

We see that the convergence criterion of $\beta(\mathcal{F}_t) \leq \prod_{k=1}^t \rho_k$ in the control setting is the same as that for the policy evaluation.
In fact, the only additional assumption we needed is the greedy-in-the-limit target policies.
Crucially, the proof allows arbitrary behavior policies and arbitrary initialization of the Q-function, which we discuss further in the next section.

\section{Discussion}
\label{sect:discussion}

Despite weighting TD errors by general coefficients $\beta(\mathcal{F}_t) \in [0, \prod_{k=1}^t \rho_k]$
that could be chosen by any arbitrary selection criteria,
the $\mathcal{M}$ operator converges for both policy evaluation and control with few assumptions.
In this section, we summarize our theoretical contributions and discuss the choice of coefficients, practical implementations, and the convergence of specific algorithms from Section~\ref{sect:example_algorithms}.

\subsection{Theoretical Contributions}
\label{sect:theoretical_contributions}

\textbf{Removal of Markov assumption.}
As we stated in the introduction, removing the Markov assumption of the $\mathcal{R}$ operator \cite{munos2016safe} was our primary theoretical goal.
With Markov traces, the operator $B_t$ is independent of $t$, allowing the sum
$C = \sum_{t \geq 0} \gamma^t P_\mu^t B_t $
to be reduced to
$\sum_{t \geq 0} \gamma^t P_{c\mu}^t$
for the linear operator
$P_{c\mu} \coloneqq \sum_{(s',a') \in \mathcal{S} \times \mathcal{A}} P(s'|s,a) \mu(a'|s') c(s',a') Q(s',a')$.
The resulting geometric series can then be evaluated analytically:
$(I - \gamma P_{c\mu})^{-1}$.
This technique was used by \cite{munos2016safe}.
In our proofs, we avoided the Markov assumption by directly analyzing the infinite summation $C$, which generally does not have a closed-form expression.
We believe our work is the first to do this, and it should have far-reaching consequences for existing on- and off-policy algorithms (see Sections~\ref{sect:convergence_for_specific} and \ref{sect:conclusion}) as well as yet-undiscovered algorithms.

\textbf{Arbitrary initialization of Q-function.}
Our proof of Theorem~\ref{theorem:control} permits any initialization of $Q_0$ in the control setting.
In contrast, \cite{munos2016safe} made the assumption that
$T_{\pi_0} Q_0 - Q_0 \geq 0$
in order to produce a lower bound on $\mathcal{R}Q_{k+1} - Q^*$.
This is accomplished in practice by a pessimistic initialization of the Q-function:
$Q_0(s,a) = - \norm{R} \mathbin{/} (1-\gamma)$,
$\forall (s,a) \in \mathcal{S} \times \mathcal{A}$.
Since $\mathcal{R}$ is a special case of our operator $\mathcal{M}$ where each $\beta(\mathcal{F}_t)$ factors into Markov traces, we deduce as a corollary that Retrace and all other algorithms described by $\mathcal{R}$ do not require this pessimistic initialization for convergence.

\textbf{Greedy-in-the-limit policies.}
Our requirement of greedy-in-the-limit target policies for the control setting is significantly less restrictive than the increasingly greedy policies proposed by \cite{munos2016safe}.
We need only that
$\lim_{k \to \infty} T_{\pi_k} Q_k = T Q_k$,
and we do not force the sequence of target policies to satisfy
$P_{\pi_{k+1}} Q_{k+1} \geq P_{\pi_k} Q_{k+1}$.
This implies that the agent may target non-greedy policies for an arbitrarily long period of time, as long as the policies do eventually become arbitrarily close to the greedy policy.
This could lead to faster learning since targeting greedy policies causes excessive trace cutting (e.g.\ Watkins' Q($\lambda$)).
Our greedy-in-the-limit policies are closely related to the Greedy-in-the-Limit with Infinite Exploration (GLIE) assumption of \cite{singh2000convergence};
however, we are able to remove the need for infinite exploration by allowing an arbitrary sequence of behavior policies, as did \cite{munos2016safe}.
Once again, our theory implies as a corollary that neither increasingly greedy policies nor the GLIE assumption are strictly necessary for the convergence of Retrace and other Markov algorithms.

\begin{algorithm}[t]
    \caption{Non-Markov Eligibility Trace}
    \label{algo:non-markov_etrace}
    \begin{algorithmic}
        \State Initialize $Q(s,a)$ arbitrarily
        \State Select learning rate $\alpha \in (0,1]$
        \For{each episode}
            \State Initialize environment state $s_0$
            \State $\mathcal{F} \gets \{\}$
            \Repeat{\ for $t = 0, 1, 2, \dots$}
                \State Take action $a_t \sim \mu(s_t)$, observe reward $r_t$ and next state $s_{t+1}$
                \State $\mathcal{F} \gets \mathcal{F} \cup (s_t, a_t)$
                \State $\delta^\pi_t \gets r_t + \gamma \sum_{a' \in \mathcal{A}} Q(s_{t+1}, a') - Q(s_t, a_t)$
                \State For $k = 0, \ldots, t$:
                    $Q(s_k,a_k) \gets Q(s_k,a_k) + \alpha \gamma^{t-k} \beta(\mathcal{F}_{k:t}) \delta^\pi_t$
                \Comment Let
                    $\mathcal{F}_{k:t} \coloneqq \mathcal{F}_t \setminus \mathcal{F}_{k-1}$,
                    $\beta(\mathcal{F}_{t:t}) \coloneqq 1$
            \Until{$s_{t+1}$ is terminal}
        \EndFor
    \end{algorithmic}
\end{algorithm}

\subsection{Practical Concerns}
\label{sect:practical_concerns}

\textbf{Implementation with eligibility traces.}
The fact that the coefficient $\beta(\mathcal{F}_t)$ does not generally factor into traces $c(s_1,a_1), \dots, c(s_t,a_t)$ means that eligibility traces in the traditional sense are not possible.
The problem occurs when the same state-action pair appears more than once in $\mathcal{F}_t$;
afterwards, there does not exist a constant factor by which to decay the eligibility for that state-action pair that would produce the correct updates according to (\ref{eq:operator}).
We can rectify this by modifying the eligibility trace to remember each individual occurrence of the state-action pairs:
a Non-Markov Eligibility Trace (see Algorithm~\ref{algo:non-markov_etrace}).
By tracking a separate eligibility for each repeated visitation, the algorithm is able to generate the correct coefficients $\beta(\mathcal{F}_t)$ even though they do not factor into traces.
In terms of computation, our algorithm is efficient when episodes are not extremely long, and is reminiscent of the practical linked-list implementation of the eligibility trace described by \cite{sutton1998reinforcement}.

\textbf{Online convergence.}
Algorithm~\ref{algo:non-markov_etrace} is presented as an online eligibility-trace algorithm.
In practice, we expect that the updates will be calculated and accumulated offline, akin to the offline $\lambda$-return algorithm \cite{sutton1998reinforcement}, for the sake of efficiency and compatibility with deep function approximation \cite{daley2019reconciling}.
As such, the analysis of the online variant is beyond the scope of our work and we leave its convergence as an open question.
Nevertheless, the fact that the expected $\mathcal{M}$ operator is convergent would make it surprising (although not impossible) that the online version diverges in the presence of zero-mean, finite-variance noise with appropriately annealed stepsizes \cite{bertsekas1996neuro}.

\textbf{Choice of coefficients $\beta(\mathcal{F}_t)$.}
Our theorems guarantee convergence provided that
$\beta(\mathcal{F}_t) \leq \prod_{k=1}^t \rho_k$
for every $\mathcal{F}_t$.
Of course, not all choices of $\beta$ will perform well in practice.
At one extreme, when
$\beta(\mathcal{F}_t) = \prod_{k=1}^t \rho_k$
for every $\mathcal{F}_t$, we recover the standard IS estimator, which we know suffers from high variance and is often ineffectual.
At the other extreme, if
$\beta(\mathcal{F}_t) < \prod_{k=1}^t \pi(a_k|s_k)$
for every $\mathcal{F}_t$, then we have a method that cuts traces more than Tree Backup does;
it seems unlikely that such a method would learn faster than Tree Backup, making it difficult to justify the extra complexity of the non-Markov coefficients in this case.

We therefore see that while it is important to control the overall variance of the coefficients $\beta(\mathcal{F}_t)$, it is also important to maintain some minimum efficiency by avoiding unnecessary trace cuts, and to leverage the non-Markov capabilities of the operator.
Effective algorithms will likely compute the full IS ratio product $\prod_{k=1}^t \rho_k$ first and then apply some nonlinear transformation (e.g.\ clipping) to control the variance.
This ensures that the coefficients are cut only when necessary.

One final yet pertinent consideration is the discounted sum of coefficients
$\sum_{t \geq 0} \gamma^t \beta(\mathcal{F}_t)$.
Roughly speaking, this sum represents the potential magnitude of an update to a state-action pair.
If its value is extremely large along some trajectories with nonzero probability of occurrence, then the algorithm may not be stable.
This suggests that a constant bound on each $\beta(\mathcal{F}_t)$, or a recency heuristic that guarantees $\beta(\mathcal{F}_t) \geq \beta(\mathcal{F}_{t+1})$, may be desirable to control the update magnitude.

\subsection{Convergence for Specific Algorithms}
\label{sect:convergence_for_specific}

In Section~\ref{sect:example_algorithms}, we mentioned some algorithms that cannot be modeled by the $\mathcal{R}$ operator can, in fact, be modeled by the $\mathcal{M}$ operator.
It remains to show that their coefficients always satisfy our condition of
$\beta(\mathcal{F}_t) \leq \prod_{k=1}^t \rho_k$
to prove that they converge according to Theorems~\ref{theorem:evaluation} and~\ref{theorem:control}.

\textbf{Truncated IS.}
Clearly,
$\min (d_t, \prod_{k=1}^t \rho_k ) \leq \prod_{k=1}^t \rho_k$ for any choice of $d_t$, hence the algorithm converges.

\textbf{Non-Markov Retrace.}
\cite{munos2016safe} define the traces for the algorithm as
$c(s_t,a_t) = \lambda \min(1 \mathbin{/} \beta(\mathcal{F}_{t-1}), \rho_t)$,
so
$\beta(\mathcal{F}_{t-1}) = \prod_{k=1}^{t-1} c(s_k,a_k)$.
This means that
\begin{equation*}
    \beta(\mathcal{F}_t)
    = \beta(\mathcal{F}_{t-1}) c(s_t,a_t)
    = \lambda \min\left( \frac{1}{\beta(\mathcal{F}_{t-1})}, \rho_t \right) \beta(\mathcal{F}_{t-1})
    = \lambda \min\left( 1, \beta(\mathcal{F}_{t-1}) \rho_t \right)
    .
\end{equation*}
We can now prove by induction that the required bound always holds.
Assume that
$\beta(\mathcal{F}_{t-1}) \leq \prod_{k=1}^{t-1} \rho_k$.
Our base case is $\beta(\mathcal{F}_0) = 1$ by definition, which clearly holds.
Then, by hypothesis,
\begin{align*}
    \beta(\mathcal{F}_t)
    = \lambda \min\left( 1, \beta(\mathcal{F}_{t-1}) \rho_t \right)
    \leq \lambda \min\left( 1, \prod_{k=1}^t \rho_k \right)
    \leq \prod_{k=1}^t \rho_k
    ,
\end{align*}
and the non-Markov variant of Retrace therefore converges.

\section{Conclusion}
\label{sect:conclusion}

Although per-decision traces are convenient from a computational perspective, they are not expressive enough to implement arbitrary off-policy corrections, and their inability to reverse previously cut traces can lead to inefficient learning.
By removing the assumption of Markov (i.e.\ per-decision) traces, our $\mathcal{M}$ operator enables off-policy algorithms that jointly consider all of the agent's past experiences when assigning credit.
This provides myriad opportunities for efficient off-policy algorithms and offers the first convergence guarantees for many previous algorithms.

The $\mathcal{M}$ operator is convergent for both policy evaluation and control.
In the latter setting, our proof removes the assumptions of increasingly greedy policies and pessimistic initialization.
This has significant implications for Retrace, the other algorithms discussed in Section~\ref{sect:example_algorithms}, Watkins' Q($\lambda$), and more.

The generality of the $\mathcal{M}$ operator means that it provides convergence guarantees for a number of TD methods that we were not able to discuss in the main text.
These include methods with variable discount factors or $\lambda$-values
\cite[e.g.][]{watkins1989learning, sutton1998reinforcement, maei2010gq, sutton2014new, mahmood2017multi, sutton2018reinforcement},
and even history-dependent $\lambda$-values
\cite[e.g.][]{singh1996reinforcement, yu2018generalized}.
All of these can be expressed in a common form:
$\beta(\mathcal{F}_t) = \prod_{k=1}^t \gamma(\mathcal{F}_k) \lambda(\mathcal{F}_k) \rho_k$
where
$\gamma(\mathcal{F}_k) \in [0,1]$
and
$\lambda(\mathcal{F}_k) \in [0,1]$.
Clearly,
$\beta(\mathcal{F}_t) \leq \prod_{k=1}^t \rho_k$
and we determine that these methods converge.

An interesting variant of our theory is one where the coefficient-determining function $\beta$ is not fixed but instead varies with time.
In this case, the operator $\mathcal{M}$ becomes non-stationary;
each diagonal matrix $B_t$ in equation (\ref{eq:operator}) depends not only on the number of timesteps $t$ since the action was taken, but also the time $t'$ when the action took place.
Such instances may arise, for example, when coefficients depend on information external to the trajectory $\mathcal{F}_t$, such as the number of previous visitations to the state-action pair \cite[e.g.][]{sutton1994step}.
Provided that the same conditions of
$B_t \leq \prod_{k=1}^t \rho_k$
hold for all $t'$, then the proofs of Theorems~\ref{theorem:evaluation} and~\ref{theorem:control} should go through with no modifications.

Important open questions remain.
Namely, do the convergence guarantees established here hold for online algorithms?
And, more practically, what non-Markov algorithms exist that lead to reliably efficient and stable learning?
Interesting possibilities with regard to both of these questions are recursive methods like Non-Markov Retrace, which are easier to implement and to analyze according to our theoretical framework.
In conjunction, our earlier considerations for choosing coefficients (Section~\ref{sect:practical_concerns}) are pertinent.
These should be interesting starting points for future developments in off-policy learning.

\small{
\bibliography{references}}

\begin{thebibliography}{10}

\bibitem{barto1983neuronlike}
Andrew~G Barto, Richard~S Sutton, and Charles~W Anderson.
\newblock Neuronlike adaptive elements that can solve difficult learning
  control problems.
\newblock {\em IEEE Transactions on Systems, Man, and Cybernetics}, pages
  834--846, 1983.

\bibitem{bellman1966dynamic}
Richard Bellman.
\newblock Dynamic programming.
\newblock {\em Science}, 153(3731):34--37, 1966.

\bibitem{bertsekas1996neuro}
Dimitri~P Bertsekas and John~N Tsitsiklis.
\newblock {\em Neuro-Dynamic Programming}.
\newblock Athena Scientific, 1996.

\bibitem{daley2019reconciling}
Brett Daley and Christopher Amato.
\newblock Reconciling $\lambda$-returns with experience replay.
\newblock In {\em Advances in Neural Information Processing Systems}, pages
  1133--1142, 2019.

\bibitem{harutyunyan2016q}
Anna Harutyunyan, Marc~G Bellemare, Tom Stepleton, and R{\'e}mi Munos.
\newblock {Q($\lambda$)} with off-policy corrections.
\newblock In {\em International Conference on Algorithmic Learning Theory},
  pages 305--320, 2016.

\bibitem{ionides2008truncated}
Edward~L Ionides.
\newblock Truncated importance sampling.
\newblock {\em Journal of Computational and Graphical Statistics},
  17(2):295--311, 2008.

\bibitem{kearns2000bias}
Michael~J Kearns and Satinder~P Singh.
\newblock Bias-variance error bounds for temporal difference updates.
\newblock In {\em COLT}, pages 142--147, 2000.

\bibitem{klopf1972brain}
A~Harry Klopf.
\newblock Brain function and adaptive systems: A heterostatic theory.
\newblock Technical report, Air Force Cambridge Research Labs, Hanscom AFB, MA,
  1972.

\bibitem{kozuno2021revisiting}
Tadashi Kozuno, Yunhao Tang, Mark Rowland, R{\'e}mi Munos, Steven Kapturowski,
  Will Dabney, Michal Valko, and David Abel.
\newblock Revisiting {P}eng's {Q($\lambda$)} for modern reinforcement learning.
\newblock {\em arXiv:2103.00107}, 2021.

\bibitem{lin1992self}
Long-Ji Lin.
\newblock Self-improving reactive agents based on reinforcement learning,
  planning and reaching.
\newblock {\em Machine Learning}, 8(3-4):293--321, 1992.

\bibitem{maei2010gq}
Hamid~Reza Maei and Richard~S Sutton.
\newblock {GQ($\lambda$)}: A general gradient algorithm for temporal-difference
  prediction learning with eligibility traces.
\newblock In {\em Proceedings of the Third Conference on Artificial General
  Intelligence}, pages 91--96, 2010.

\bibitem{mahmood2017multi}
Ashique~Rupam Mahmood, Huizhen Yu, and Richard~S Sutton.
\newblock Multi-step off-policy learning without importance sampling ratios.
\newblock {\em arXiv:1702.03006}, 2017.

\bibitem{mnih2015human}
Volodymyr Mnih, Koray Kavukcuoglu, David Silver, Andrei~A Rusu, Joel Veness,
  Marc~G Bellemare, Alex Graves, Martin Riedmiller, Andreas~K Fidjeland, Georg
  Ostrovski, et~al.
\newblock Human-level control through deep reinforcement learning.
\newblock {\em Nature}, 518(7540):529--533, 2015.

\bibitem{munos2016safe}
R{\'e}mi Munos, Tom Stepleton, Anna Harutyunyan, and Marc~G Bellemare.
\newblock Safe and efficient off-policy reinforcement learning.
\newblock {\em arXiv:1606.02647}, 2016.

\bibitem{peng1994incremental}
Jing Peng and Ronald~J Williams.
\newblock Incremental multi-step {Q-Learning}.
\newblock {\em Machine Learning}, pages 226--232, 1994.

\bibitem{precup2000eligibility}
Doina Precup, Richard~S Sutton, and Satinder Singh.
\newblock Eligibility traces for off-policy policy evaluation.
\newblock In {\em Proceedings of the Seventeenth International Conference on
  Machine Learning}, page 759–766, 2000.

\bibitem{rowland2020adaptive}
Mark Rowland, Will Dabney, and R{\'e}mi Munos.
\newblock Adaptive trade-offs in off-policy learning.
\newblock In {\em International Conference on Artificial Intelligence and
  Statistics}, pages 34--44, 2020.

\bibitem{rummery1994line}
Gavin~A Rummery and Mahesan Niranjan.
\newblock On-line {Q-Learning} using connectionist systems.
\newblock Technical report, Cambridge University, 1994.

\bibitem{singh2000convergence}
Satinder Singh, Tommi Jaakkola, Michael~L Littman, and Csaba Szepesv{\'a}ri.
\newblock Convergence results for single-step on-policy reinforcement-learning
  algorithms.
\newblock {\em Machine learning}, 38(3):287--308, 2000.

\bibitem{singh1996reinforcement}
Satinder~P Singh and Richard~S Sutton.
\newblock Reinforcement learning with replacing eligibility traces.
\newblock {\em Machine learning}, 22(1):123--158, 1996.

\bibitem{sutton1984temporal}
Richard~S Sutton.
\newblock {\em Temporal Credit Assignment in Reinforcement Learning}.
\newblock PhD thesis, University of Massachusetts Amherst, 1984.

\bibitem{sutton1988learning}
Richard~S Sutton.
\newblock Learning to predict by the methods of temporal differences.
\newblock {\em Machine learning}, 3(1):9--44, 1988.

\bibitem{sutton1998reinforcement}
Richard~S Sutton and Andrew~G Barto.
\newblock {\em Reinforcement Learning: An Introduction}.
\newblock MIT Press, 1st edition, 1998.

\bibitem{sutton2018reinforcement}
Richard~S Sutton and Andrew~G Barto.
\newblock {\em Reinforcement Learning: An Introduction}.
\newblock MIT Press, 2nd edition, 2018.

\bibitem{sutton2014new}
Richard~S Sutton, Ashique~Rupam Mahmood, Doina Precup, and Hado Hasselt.
\newblock A new {Q($\lambda$)} with interim forward view and {M}onte {C}arlo
  equivalence.
\newblock In {\em International Conference on Machine Learning}, pages
  568--576, 2014.

\bibitem{sutton1994step}
Richard~S Sutton and Satinder~P Singh.
\newblock On step-size and bias in temporal-difference learning.
\newblock In {\em Proceedings of the Eighth Yale Workshop on Adaptive and
  Learning Systems}, pages 91--96, 1994.

\bibitem{uchibe2004competitive}
Eiji Uchibe and Kenji Doya.
\newblock Competitive-cooperative-concurrent reinforcement learning with
  importance sampling.
\newblock In {\em Proceedings of International Conference on Simulation of
  Adaptive Behavior: From Animals and Animats}, pages 287--296, 2004.

\bibitem{wang2016sample}
Ziyu Wang, Victor Bapst, Nicolas Heess, Volodymyr Mnih, Remi Munos, Koray
  Kavukcuoglu, and Nando de~Freitas.
\newblock Sample efficient actor-critic with experience replay.
\newblock {\em arXiv:1611.01224}, 2016.

\bibitem{watkins1989learning}
Christopher John Cornish~Hellaby Watkins.
\newblock {\em Learning from Delayed Rewards}.
\newblock PhD thesis, King's College, Cambridge, 1989.

\bibitem{wawrzynski2009real}
Pawe{\l} Wawrzy{\'n}ski.
\newblock Real-time reinforcement learning by sequential actor-critics and
  experience replay.
\newblock {\em Neural Networks}, 22(10):1484--1497, 2009.

\bibitem{wawrzynski2007truncated}
Pawel Wawrzynski and Andrzej Pacut.
\newblock Truncated importance sampling for reinforcement learning with
  experience replay.
\newblock {\em Proc. CSIT Int. Multiconf}, pages 305--315, 2007.

\bibitem{yu2018generalized}
Huizhen Yu, A~Rupam Mahmood, and Richard~S Sutton.
\newblock On generalized {B}ellman equations and temporal-difference learning.
\newblock {\em The Journal of Machine Learning Research}, 19(1):1864--1912,
  2018.

\end{thebibliography}

\clearpage
\appendix

\section{Proofs}

\subsection{Proof of Lemma~\ref{lemma:diff}}
\label{app:lemma_diff}

\begin{proof}
    We begin by rewriting the operator $\mathcal{M}$ in (\ref{eq:operator}):
    \begin{align*}
        \allowdisplaybreaks
        \mathcal{M}Q &= Q + \sum_{t \geq 0} \gamma^t P_\mu^t B_t (T_\pi Q - Q) \\
        &= Q + \sum_{t \geq 0} \gamma^t P_\mu^t B_t (T_\pi Q - R) + \sum_{t \geq 0} \gamma^t P_\mu^t B_t (R - Q) \\
        &= R + \sum_{t \geq 0} \gamma^t P_\mu^t B_t (T_\pi Q - R) + \sum_{t \geq 1} \gamma^t P_\mu^t B_t (R - Q)
        .
    \end{align*}
    It is evident from (\ref{eq:operator}) that $Q^\pi$ is the fixed point of $\mathcal{M}$ because $Q = Q^\pi \implies T_\pi Q - Q = 0$.
    Therefore,
    \begin{align*}
        \allowdisplaybreaks
        \mathcal{M}Q - Q^\pi
        &= \mathcal{M}Q - \mathcal{M}Q^\pi \\
        &= \sum_{t \geq 0} \gamma^t P_\mu^t B_t (T_\pi Q - T_\pi Q^\pi) - \sum_{t \geq 1} \gamma^t P_\mu^t B_t (Q - Q^\pi) \\
        &= \sum_{t \geq 0} \gamma^{t+1} P_\mu^t B_t P_\pi (Q - Q^\pi) - \sum_{t \geq 1} \gamma^t P_\mu^t B_t (Q - Q^\pi) \\
        &= \sum_{t \geq 1} \gamma^t P_\mu^{t-1} B_{t-1} P_\pi (Q - Q^\pi) - \sum_{t \geq 1} \gamma^t P_\mu^t B_t (Q - Q^\pi) \\
        &= \sum_{t \geq 1} \gamma^t P_\mu^{t-1} ( B_{t-1} P_\pi - P_\mu B_t) (Q - Q^\pi)
        ,
    \end{align*}
    which is the desired quantity.
\end{proof}

\subsection{Proof of Lemma~\ref{lemma:matrix_norm_bound}}
\label{app:matrix_norm_bound}

\begin{proof}
    We start with the negative of the original matrix expression:
    $C(I - \gamma P_\pi) - I = (C - I) - \gamma C P_\pi$.
    This expression is affine in $C$;
    we can check its two extremal values to determine the norm.

    Recall that $C = \sum_{t \geq 0} \gamma^t P_\mu^t B_t$, a nonnegative matrix.
    One extreme occurs when $B_t = 0$ for $t \geq 1$ and hence $C = I$;
    substituting this into the expression, we obtain $- \gamma P_\pi$ and deduce that $\norm{C(I - \gamma P_\pi) - I} = \gamma$ when $C = I$.

    The other extreme occurs when each $B_t$ is maximized.
    By our criteria for $B_t$ from Theorem~\ref{theorem:evaluation}, we must have $P_\mu^t B_t \leq P_\pi^t$ and therefore
    $C \leq \sum_{t \geq 0} \gamma^t P_\pi^t = (I - \gamma P_\pi)^{-1}$.
    We can also deduce in this case that 
    \begin{align*}
        \allowdisplaybreaks
        C - I
        \leq (\gamma P_\pi + \gamma^2 P_\pi + \cdots)
        = \gamma (I + \gamma P_\pi + \cdots) P_\pi
        = \gamma (I - \gamma P_\pi)^{-1} P_\pi
        .
    \end{align*}
    Substituting these into the expression,
    we obtain
    $\gamma (I - \gamma P_\pi)^{-1} P_\pi - \gamma (I - \gamma P_\pi)^{-1} P_\pi = 0$
    when
    $C = (I - \gamma P_\pi)^{-1}$.

    Combining these two cases, we conclude that
    $\norm{I - C(I - \gamma P_\pi)} = \norm{C(I - \gamma P_\pi) - I} \leq \gamma$.
\end{proof}

\end{document}